\documentclass{article}

\usepackage[preprint]{neurips_2020}

\usepackage[utf8]{inputenc} 
\usepackage[T1]{fontenc}    
\usepackage{hyperref}       
\usepackage{url}            
\usepackage{booktabs}       
\usepackage{amsfonts}       
\usepackage{nicefrac}       
\usepackage{microtype}      

\usepackage{multirow}

\usepackage[toc,page]{appendix}

\usepackage{adjustbox}

\title{On Information Gain and Regret Bounds in Gaussian Process Bandits}

\author{%
  Sattar Vakili,\\
  MediaTek Research, UK \\
  \texttt{sattar.vakili@mtkresearch.com}
  \And
  Kia Khezeli\\
  Mayo Clinic, MN, USA\\
  \texttt{khezeli.kia@mayo.edu}
   \And
   Victor Picheny \\
  Secondmind, UK \\
  victor@secondmind.ai
}

\usepackage[english]{babel}
\setcitestyle{authoryear,open={(},close={)}}
\usepackage{setspace}

\usepackage[utf8]{inputenc} 
\usepackage[T1]{fontenc}    
\usepackage{hyperref}  
\hypersetup{
    colorlinks=true,
    linkcolor=blue,
    citecolor =blue,
    filecolor=magenta,      
    urlcolor=magenta,
}
\usepackage{url}            
\usepackage{booktabs}       
\usepackage{amsfonts}       
\usepackage{nicefrac}       
\usepackage{microtype}      
\usepackage{lipsum}
\usepackage{graphicx}

\providecommand{\keywords}[1]
{
  \small	
  \textbf{\textit{Keywords---}} #1
}

\usepackage{amsmath}
\usepackage{amsfonts}
\usepackage[ruled,vlined]{algorithm2e}
\usepackage{bm}
\usepackage{amssymb}
\usepackage{color}
\usepackage{xcolor}
\usepackage{hyperref}
\usepackage{amsmath, amssymb, graphicx, url, amsthm}

\usepackage{cleveref}
\crefformat{section}{\S#2#1#3} 
\crefformat{subsection}{\S#2#1#3}
\crefformat{subsubsection}{\S#2#1#3}

\usepackage{mathrsfs}
\usepackage{makecell}

\def\tr{\text{tr}}

\def\df{\mathcal{D}_T}

\def\W{\mathbf{W}}
\def\X{\mathcal{X}}

\def\H{\mathcal{H}}

\def\O{\mathcal{O}}

\def\phib{\bm{\phi}}
\def\Ib{\bm{I}}

\def\argmax{\footnotesize \mbox{argmax}}

\def\TP{\top}

\def\phib{\bm{\phi}}

\def\Phib{\bm{\Phi}}

\def\Xb{\mathbf{X}}
\def\yb{\mathbf{y}}
\def\Ib{\mathbf{I}}

\def\Rr{\mathbb{R}}
\def\Nn{\mathbb{N}}
\def\E{\mathbb{E}}

\newtheorem{lemma}{Lemma}
\newtheorem{theorem}{Theorem}

\newtheorem{corollary}{Corollary}
\newtheorem{definition}{Definition}
\newtheorem{remark}{Remark}

\newtheorem{assumption}{Assumption}

\def\nn{\nonumber}

\begin{document}
\maketitle

\begin{abstract}

    Consider the sequential optimization of an expensive to evaluate and possibly non-convex objective function $f$ from noisy feedback, that can be regarded as a continuum-armed bandit problem. Upper bounds 
    on the regret performance of several learning algorithms (GP-UCB, GP-TS, and their variants) are known under both a Bayesian (when $f$ is a sample from a Gaussian process (GP)) and a frequentist (when $f$ lives in a reproducing kernel Hilbert space) setting. The regret bounds often rely on the maximal information gain $\gamma_T$ between $T$ observations and the underlying GP (surrogate) model. 
    We provide general bounds on $\gamma_T$ based on the decay rate of the eigenvalues of the GP kernel, whose specialisation for commonly used kernels improves the existing bounds on $\gamma_T$, and subsequently the regret bounds relying on $\gamma_T$ under numerous settings.  
    For the Mat{\'e}rn family of kernels, where the lower bounds on $\gamma_T$, and regret under the frequentist setting, are known, our results close a huge polynomial in $T$ gap between the upper and lower bounds (up to logarithmic in $T$ factors).

\end{abstract}
{\keywords{Information gain, effective dimension, regret bounds, Bayesian optimization, GP-UCB, GP-TS, continuum-armed bandits.}}
\section{Introduction}

Bayesian optimization building on Gaussian Process (GP) models has been shown to efficiently address the exploration-exploitation trade-off in the sequential optimization of non-convex objective functions with bandit feedback. 
There have been significant recent advances in the analysis of GP-based Bayesian optimization algorithms, providing performance guarantees in terms of regret.
Regret is defined as the cumulative loss in the value of the objective function $f$ at a sequence of observation points $\{x_t\}_{t=1}^T$, $T\in \Nn$, in comparison to its value at a global maximum $x^*\in\text{argmax}_{x\in\X} {f(x)}$ over the search space $\X\in \Rr^d$  (see~\eqref{eq:regdef}).
In their seminal paper, ~\cite{srinivas2010gaussian} established performance guarantees for GP-UCB, an \emph{optimistic} optimization algorithm which sequentially selects the $x_t$ that maximize an {u}pper confidence bound score over the search space. 
They considered a fully Bayesian setting where $f$ is assumed to be a sample from a GP with a known kernel, as well as a frequentist setting (referred to as agnostic in~\cite{srinivas2010gaussian}) where $f$ is assumed to live in a reproducing kernel Hilbert space (RKHS) with a known kernel.
They showed an $\tilde{\O}(\sqrt{\gamma_T T})$\footnote{We use the notations $\O$ and $\Omega$ to denote the standard mathematical orders and the notation $\tilde{\O}$ to suppress the logarithmic factors.} and an $\tilde{\O}(\gamma_T\sqrt{T})$
regret bound for GP-UCB under the Bayesian and frequentist settings, respectively, where $\gamma_T$ is the {maximal information gain} between the observed sequence and the underlying model~(see~\cref{InfoGainRegBounds}).
Sublinearly scaling with $T$, $\gamma_T$ depends on the underlying kernel and is interpreted as a measure for the difficulty of the optimization task. Since the pioneering work of ~\citeauthor{srinivas2010gaussian}, there have been several results on improving the bounds toward their optimal value. 
\cite{Chowdhury2017bandit} improved the regret bounds under the frequentist setting by multiplicative logarithmic in $T$ factors. 
A variant of GP-UCB (called SupKernelUCB), which builds on episodic independent batches of observations, was shown in~\cite{Valko2013kernelbandit} to achieve the $\tilde{\O}(\sqrt{\gamma_T T})$ regret bound under the frequentist setting.
Furthermore,~\cite{Chowdhury2017bandit} showed that $\tilde{O}(\gamma_T\sqrt{T})$ regret bounds, under the frequentist setting, also hold for GP-TS, a Bayesian optimization algorithm based on Thompson Sampling which sequentially draws $x_t$ from the posterior distribution of $x^*$. Under the Bayesian setting,~\cite{kandasamy2018parallelised} built on ideas from~\cite{Russo2014, Russo2016Info} to show that GP-TS achieves the same order of regret as GP-UCB.


The regret bounds mentioned above become complete only when $\gamma_T$ is properly bounded, which proves challenging. We first overview the existing upper and lower bounds on the maximal information gain in the literature. We then discuss our contribution and compare it against the state of the art.

\subsection{Upper and Lower Bounds on $\gamma_T$}

\cite{srinivas2010gaussian} showed that $\gamma_T= \tilde{\O}(T^{\frac{d(d+1)}{2\nu+d(d+1)}})$  
for the Mat{\'e}rn-$\nu$ kernel (a Mat{\'e}rn kernel with smoothness parameter $\nu$; see~\cref{Sec:GPs} for the details), and $\gamma_T = \O(\log^{d+1}(T))$ for the Squared Exponential (SE) kernel. Recently,
\cite{Janz2020SlightImprov} introduced a GP-UCB based algorithm (specific to the Mat{\'e}rn family of kernels), that constructs a cover for the search space (as many hypercubes) and fits an independent GP to each cover element. This analysis elicited an improved bound for the Mat{\'e}rn-$\nu$ kernel; $\gamma_T= \tilde{\O}(T^{\frac{d(d+1)}{2\nu+d(d+\bm{2})}})$. 
Plugging these bounds on $\gamma_T$ into the $\tilde{\O}(\sqrt{\gamma_T T})$ regret bounds mentioned above yields explicit upper bounds, in terms of $T$, which are in the forms of $\tilde{\O}(T^\frac{\nu+d(d+1.5)}{2\nu+d(d+2)})$ and $\tilde{\O}(T^{\frac{1}{2}}\log^{\frac{d}{2}+1}(T))$, with Mat{\'e}rn-$\nu$ and SE, respectively.



Finding the order optimal regret bounds is a long standing open question. 
Under the frequentist setting,~\cite{Scarlett2017Lower} proved the $\Omega(T^{\frac{\nu+d}{2\nu+d}})$ and $\Omega(T^{\frac{1}{2}}\log^{\frac{d}{2}}(T))$ lower bounds on the regret performance of any learning algorithm, with Mat{\'e}rn-$\nu$ and SE, respectively. From the results of~\cite{Scarlett2017Lower} and~\cite{Valko2013kernelbandit}, $\Omega(T^{\frac{d}{2\nu+d}})$ and $\Omega(\log^{d}(T))$ lower bounds on $\gamma_T$ can be concluded for Mat{\'e}rn-$\nu$ and SE, respectively~\citep[see also][]{Janz2020SlightImprov}, which facilitate the assessment of the upper bounds. 

While the bounds are tight up to logarithmic factors for SE,
a comparison between the lower and upper bounds for the practically useful Mat{\'e}rn family of kernels
shows a drastic gap, which can surprisingly be as large as $\O(\sqrt{T})$ in the case of regret, and as large as $\O(T)$ in the case of $\gamma_T$, with particular configurations of parameters $\nu$ and $d$.\footnote{Consider a case where $\nu$ and $d$ grow large; $\nu$ grows faster than $d$ and slower than $d^2$. Then, the lower bound and the upper bound on regret become arbitrarily close to $\Omega(\sqrt{T})$ and $\O(T)$, respectively. Therefore, the worst case gap between them is in $\O(\sqrt{T})$. The lower bound and the upper bound on $\gamma_T$ become arbitrarily close to $\Omega(1)$ and $\O(T)$, respectively. Therefore, the worst case gap between them is in $\O({T})$.} Motivated by this huge gap in the literature, we aim to provide tight bounds on $\gamma_T$ (reducing these polynomial in $T$ gaps to logarithmic ones), as outlined in the next section.
\subsection{Contribution}

Our contribution is in establishing novel bounds on $\gamma_T$, which directly translate to new regret bounds for Bayesian optimization algorithms. 
To achieve this, we use Mercer's theorem to represent the GP kernel in terms of its eigenvalue-eigenfeature decomposition ---an inner product in the corresponding reproducing kernel Hilbert space (RKHS)--- which is infinite dimensional for typical kernels. To overcome the difficulty of working in infinite dimensional spaces, we use a projection on a finite $D$ dimensional space that allows us to bound the information gain in terms of $D$ and the spectral properties of the GP kernel. 
For a kernel with decreasing eigenvalues $\{\lambda_m\}_{m=1}^\infty$, we consider two cases of polynomial, $\lambda_m=\O(m^{-\beta_p}) $, $\beta_p>1$, and exponential, $\lambda_m=\O(\exp(-m^{\beta_e})) $, $\beta_e>0$, decays. We prove $\O\left(T^{\frac{1}{\beta_p}}\log^{1-\frac{1}{\beta_p}}(T)\right)$ and $\O\left(\log^{1+\frac{1}{\beta_e}}(T)\right)$ upper bounds on $\gamma_T$ under these two cases, respectively. The application of our bounds on $\gamma_T$ to the regret bounds results in new upper bounds based on $\beta_p$ and $\beta_e$ which are summarized in Table~\ref{TablMatSE}.
In comparison to the existing works, which rely on specific kernels (e.g. Mat{\'e}rn-$\nu$ and SE) for explicit regret bounds, our results provide general explicit regret bounds, providing the conditions on the decay rate of the eigenvalues of the GP kernel (referred to as eigendecay for brevity) are satisfied.

 
 
 

As an instance of polynomially decaying eigenvalues, our results apply to the Mat{\'e}rn-$\nu$ kernel (see~\cref{Sec:GPs}) showing $\tilde{\O}(T^{\frac{d}{2\nu+d}})$ and $\tilde{\O}(T^{\frac{\nu+d}{2\nu+d}})$ bounds on $\gamma_T$ and regret, respectively. Our bounds on $\gamma_T$ and regret (under the frequentist setting) are tight, closing the gap with the respective $\Omega(T^{\frac{d}{2\nu+d}})$ and ${\Omega}(T^{\frac{\nu+d}{2\nu+d}})$ lower bounds reported in~\cite{Scarlett2017Lower,Janz2020SlightImprov}, both up to logarithmic factors. 
As an instance of exponentially decaying eigenvalues, our results apply to the Squared Exponential (SE) kernel. A summary of the results is given in Table~\ref{TablMatSE}.

\begin{table}[ht]
\centering
\begin{adjustbox}{width = \textwidth}
\begin{tabular}{ c c c c c c c }  
 \Xhline{2\arrayrulewidth}
 {\small Kernel}&& {\small Bound on $\gamma_T$}&& {\small Regret Lower Bound}&& {\small Regret Upper Bound $(\tilde{\O}(\sqrt{\gamma_T T}))$} \\
 \cline{1-1}
 \cline{3-3}
 \cline{5-5}
 \cline{7-7}
 &&&&&&\\
 Polynomial eigendecay  && {\small$\O\left(T^{\frac{1}{\beta_p}}\log^{1-\frac{1}{\beta_p}}(T)\right)$} 
  && -&&{\small$\tilde{\O}\left(T^{\frac{\beta_p+1}{2\beta_p}}\right)$}  
  \\ 
Exponential eigendecay&&{\small$\O\left(\log^{1+\frac{1}{\beta_e}}(T)\right)$}&&-&&{\small$\tilde{\O}\left(T^{\frac{1}{2}}\log^{\frac{1}{2\beta_e}}(T)\right)$} \\
{\small Mat{\'e}rn$-\nu$}&&{\small$
 {\O}\left(
T^{\frac{d}{2\nu+d}}\log^{\frac{2\nu}{2\nu+d}}(T)\right)$}&&{\small$\Omega(T^{\frac{\nu+d}{2\nu+d}})$}&& {\small$\tilde{\O}\left(
T^{\frac{\nu+d}{2\nu+d}}
\right)$}\\

 {\small SE}  &&  {\small$\O\left( \log^{d+1}(T) \right)$}&&  {\small$\Omega\left(T^{\frac{1}{2}}\log^{\frac{d}{2}}(T)\right)$}&&  {\small$\tilde{\O}\left(T^{\frac{1}{2}}\log^{\frac{d}{2}}(T)\right)$} \\ 

 &&&&&&\\
\Xhline{2\arrayrulewidth}
\end{tabular}
\end{adjustbox}
\vspace{1em}
\caption{{\small 
The Upper bounds on the maximal information gain $\gamma_T$ and the regret of Bayesian optimization algorithms under general polynomial and exponential conditions on the eigendecay of the GP kernel (see Definition~\ref{Def:PolExp}), as well as, with Mat{\'e}rn-$\nu$ and SE kernels (established in this paper). The lower bounds on regret under the frequentist setting (on the third column of the table) were reported in~\cite{Scarlett2017Lower}. 
The gap between the upper and lower bounds is reduced to logarithmic factors. 
}}\label{TablMatSE}
\end{table}

While we focus on the standard sequential optimization problem in this paper, it is worth noting that the bounds on $\gamma_T$ are also essential for numerous variants of the problem such as the ones under the settings with contextual information, safety constraints and multi-fidelity evaluations (see~\cref{RelatedW} for a list of references). Our bounds on $\gamma_T$ directly apply and improve the regret bounds depending on $\gamma_T$ under these various settings.

The interest in the bounds on $\gamma_T$ goes beyond the regret bounds. 
For example, the confidence bounds for the RKHS elements~\citep[see e.g.,][Theorem 2]{Chowdhury2017bandit} depend on $\gamma_T$. Another closely related quantity is the so called 
effective dimension $\tilde{\mathcal{D}}_T$ of the problem that satisfies $\tilde{\mathcal{D}}_T =\O(\gamma_T)$~\citep[see][ and Remark~\ref{remark1}]{Valko2013kernelbandit, Calandriello2019Adaptive, Janz2020SlightImprov}. 
\cite{Calandriello2019Adaptive} introduced a variation of GP-UCB which improves its computational cost. The improved computational cost depends on $\tilde{\mathcal{D}}_T$.
Our bounds on $\gamma_T$ (consequently on $\tilde{\mathcal{D}}_T$) improve such bounds on the algorithmic properties of GP-based methods.

\subsection{Other Related Work}\label{RelatedW}


Recent years have shown an increasing interest in Bayesian optimization based on GP models. Performance guarantees in terms of regret are studied under various settings including contextual information~\citep{Krause11Contexual}, high dimensional spaces~\citep{Josip13HighD, Mutny2018SGPTS}, safety constraints~\citep{berkenkamp2016bayesiansafe, sui2018stagewisesafe}, parallelization~\citep{kandasamy2018parallelised}, 
multi-fidelity evaluations~\citep{kandasamy2019multifidelity}, ordinal models~\citep{picheny2019ordinal}, online control~\citep{NEURIPS2020_aee5620f}, and corruption tolerance~\citep{bogunovic2020corruption}, to name a few.
\cite{Javidi} introduced an adaptive discretization of the search space improving the computational complexity of a GP-UCB based algorithm. Sparse approximation of GP posteriors are shown to preserve the regret orders while significantly improving the computational complexity of both GP-UCB~\citep{Mutny2018SGPTS, Calandriello2019Adaptive} and GP-TS~\citep{Vakili2020Scalable}. Most of existing work reports regret bounds in terms of $\gamma_T$. Our results directly apply to, and improve, the regret bounds in
all of the works mentioned above, should our bounds on $\gamma_T$ replace the existing ones. 

Our analytical approach and conditions on the eigendecay of GP kernels bear similarity to~\cite{Bartlett2018}, where the authors studied the problem of online learning with kernel losses. The problems and their analysis, however, hold substantial differences. A more challenging adversarial setting was considered for the objective function in~\cite{Bartlett2018}. However, the objective function was restricted to the subspace of one dimensional functions in the RKHS, which is very limiting for our purposes (one of the main challenges in our analysis is the infinite-dimensionality of the RKHS). The algorithmic designs, based on exponential weights, under the adversarial setting, are also significantly different from GP-UCB and GP-TS, especially, in the sense that their analysis does not rely on the information gain.

Instead of Mercer's Theorem, other decompositions of GP kernels may also be used in a similar way to our analysis. For instance, decompositions based on Fourier features were used in~\cite{Mutny2018SGPTS} to implement computationally efficient variations of GP-TS and GP-UCB. They did not however consider the analysis of $\gamma_T$.

Under the Bayesian setting,~\cite{scarlett2018tight} proved tight $\Omega(\sqrt{T})$ (up to a logarithmic in $T$ factor) lower bounds on regret when the search space is one-dimensional ($d=1$). To the best of our knowledge, lower bounds are unknown for the general case ($d>1$), under the Bayesian setting. 

Both GP-UCB and GP-TS are rooted in the classic multi-armed bandit literature \citep[see][and references therein]{Aeur2002UCB, Russo2016Info,slivkins2019introduction,zhao2019book}. Our work strengthens the link between linear~\citep{Dani2008,rusmevichientong2010linearly,Abbasi2011,agrawal2013thompson, Abeille2017LinTS} and kernelized (GP-based)~\citep{srinivas2010gaussian,Chowdhury2017bandit} models for sequential optimization with bandit feedback, as we build our analysis based on a finite-dimensional projection that is equivalent to linear bandits.

\vspace{1em}
The remainder of the paper is organized as follows.
The problem formulation, the preliminaries on GPs, GP-UCB, GP-TS, and the background on the connection between the regret bounds and the information gain are presented in~\cref{Sec:PF}. The analysis of the bounds on $\gamma_T$ is provided in~\cref{Sec:Analysis}. The explicit regret bounds (in terms of $T$) for Bayesian optimization algorithms are given in~\cref{Sec:RegBounds}. 
The paper is concluded in~\cref{Sec:Discussion}. 
\section{Problem Formulation and Preliminaries}\label{Sec:PF}
In this section, we provide background information on sequential optimization, GPs, and the connection between the information gain and the regret bounds for Bayesian optimization algorithms.

We use the following notations throughout the paper. 
For a square matrix $M\in \Rr^{n\times n}$, the notations $\det(M)$ and $\tr(M)$ denote the determinant and the trace of $M$, respectively. The notation $M^{\TP}$  is used for the transpose of an arbitrary matrix $M$. For a positive definite matrix $P$, $\log\det(P)$ denotes $\log(\det(P))$. The identity matrix of dimension $n$ is denoted by $\Ib_n$. For a vector $z\in \Rr^n$,
the notation $\|z\|_2$ denotes its $l^2$ norm.

\subsection{The Sequential Optimization Problem}\label{SeqOptP}

Consider the sequential optimization of a fixed and unknown objective function $f$ over a compact set $\X\subset \Rr^d$. A learning algorithm $\pi$ sequentially selects an observation point $x_t\in \X$ at each discrete time instance $t=1,2,\dots$, and receives the corresponding real-valued reward $y_t = f(x_t)+\epsilon_t$, where $\epsilon_t$ is the observation noise. Specifically, $\pi=\{\pi_t\}_{t=1}^{\infty}$ is a sequence of mappings $\pi_t:\H_{t-1}\rightarrow \X$ from the history of observations to a new observation point; 
$\H_t=\{\Xb_t,\yb_t\}$, $\Xb_t = [x_1,x_2,...,x_t]^{\TP}$, $\yb_t = [y_1,y_2,...,y_t]^{\TP}$, $x_s\in\X$, $y_s\in \Rr$, for all $s\ge 1$.
The regularity assumptions on $f$ and $\epsilon_t$ are specified in~\cref{Regularity}.

The goal is to minimize regret, defined as the cumulative loss compared to the maximum attainable objective, over a time horizon $T$. Specifically, 
\begin{eqnarray}\label{eq:regdef}
R(T;\pi) = \sum_{t=1}^T \left(f(x^*) - f(x_t)\right),
\end{eqnarray}
where $x^* {\in}\argmax_{x\in\X}f(x)$ is a global maximum of $f$.
To simplify the notation, 
the dependency on $\pi$ in the notation of $\Xb_t$ has been omitted.

\subsection{Gaussian Processes}\label{Sec:GPs}

The learning algorithms considered here build on GP (surrogate) models. 
A GP is a random process $\{{\hat{f}}(x)\}_{x \in \X}$, in which all finite subsets follow multivariate Gaussian distributions \citep{Rasmussen2006}. The distribution of a GP can be specified by its mean function $\mu(x)=\E[\hat{f}(x)]$ and a positive definite kernel (or covariance function) $k(x,x') = \E\left[(\hat{f}(x)-\mu(x))(\hat{f}(x')-\mu(x'))\right]$. Without loss of generality, it is typically assumed that $\forall x\in\X,\mu(x)=0$ for prior GP distributions.

Conditioning GPs on available observations provides us with powerful non-parametric Bayesian (surrogate) models over the space of functions.
In particular, conditioned on $\H_{t}$, the posterior of $\hat{f}$ is a GP with mean function $\mu_{t}(x)= \E[\hat{f}(x)|\H_{t}] $ and kernel function $k_t(x,x') = \E[(\hat{f}(x)-\mu_t(x))(\hat{f}(x')-\mu_t(x'))|\H_{t}]$ specified as follows:
\begin{eqnarray}\nn
\mu_t(x) &=& k^{\TP}_{\Xb_t,x} (K_{\Xb_t,\Xb_t}+\tau \Ib)^{-1} \yb_{t}, \\\nn
k_{t}(x,x') &=&   k(x,x') -  k^{\TP}_{\Xb_t,x} (K_{\Xb_t,\Xb_t}+\tau \Ib)^{-1} k_{\Xb_t,x'},
\end{eqnarray}

where $k_{\Xb_t,x} = \left[~k(x_1,x),k(x_2,x), \dots, k(x_{t},x)~\right]^{\TP}$ and $K_{\Xb_t,\Xb_t}$ is the ${t}\times{t}$ positive definite covariance matrix, $[k(x_i,x_j)]_{i,j=1}^{t}$. The posterior variance of $\hat{f}(x)$ is denoted by $\sigma^2_t(x) = k_t(x,x)$.  

Mat{\'e}rn and squared exponential (SE) are perhaps the most popular kernels in practice for Bayesian optimization~\citep[see e.g.,][]{Snoek2012practicalBO,Shahriari2016outofloop},
\begin{eqnarray}\nn
\scriptsize
k_{\text{Mat{\'e}rn}}(x,x') &=& \frac{1}{\Gamma(\nu)2^{\nu-1}}\left(\frac{\sqrt{2\nu}r}{l}\right)^{\nu}B_{\nu}\left(\frac{\sqrt{2\nu}r}{l}\right),\\\nn
\scriptsize k_{\text{SE}}(x,x') &=& \exp \left(-\frac{r^2}{2l^2} \right),
\end{eqnarray}
where $l >0$, $r=\|x-x'\|_2$ is the Euclidean distance between $x$ and $x'$,  $\nu>0$ is referred to as the smoothness parameter, $\Gamma$ is the gamma function, and $B_\nu$ is the modified Bessel function of the second kind. Variation over parameter $\nu$ creates a rich family of kernels.
The SE kernel can also be interpreted as a special case of Mat{\'e}rn family when $\nu\rightarrow\infty$. 

\subsection{Bayesian Optimization Algorithms (GP-UCB and GP-TS)}

GP-UCB relies on an \emph{optimistic} upper confidence bound score to select the observation points. Specifically, at each time $t$, $x_t$ is selected as
\begin{eqnarray}\nn
x_t =\text{argmax}_{x\in \X} \mu_{t-1}(x)+\alpha_{t}\sigma_{t-1}(x),
\end{eqnarray}
where $\mu_{t-1}$ and $\sigma_{t-1}$ are the posterior mean and the standard deviation, based on previous observations defined in~\cref{Sec:GPs}, and $\alpha_t$ is a  user-specified scaling parameter. 

GP-TS selects the observation points by posterior sampling. Specifically, at each time $t$, a sample $\hat{f}_t(x)$ is drawn from a GP with mean $\mu_{t-1}$ and kernel function $\alpha_t^2k_{t-1}$ where  $\mu_{t-1}$ and $k_{t-1}$ are the posterior mean and the posterior kernel, based on previous observations defined~in~\cref{Sec:GPs}, and $\alpha_t$ is a user-specified scaling parameter. Then, $x_t$ is selected as
\begin{eqnarray}
x_t = \text{argmax}_{x\in \X}\ \hat{f}_t(x).
\end{eqnarray}

The scaling parameters $\alpha_t$ are designed to balance the trade-off between exploitation and exploration of the search space and increase with $t$ ($\alpha_t>\alpha_{t'}$ when $t>t'$). See, e.g., \cite{srinivas2010gaussian, Chowdhury2017bandit} for the specifications of $\alpha_t$.

\subsection{Regularity Assumptions}\label{Regularity}

The regret performance of the learning algorithms is analysed under two different settings, referred to as Bayesian and frequentist. 

Under the {Bayesian} setting, $f$ is assumed to be a sample from a prior GP with kernel $k$. The observation noise $\{\epsilon_t\}_{t=1}^T$ are assumed to be i.i.d. zero mean Gaussian random variables with variance~$\tau$. 

Under the {frequentist} setting, $f$ is assumed to live in the RKHS corresponding to $k$. In particular,
$\|f\|_{H_k}\le B$, for some $B>0$, where $\|\cdot\|_{H_k}$ denote the RKHS norm~(see~\cref{RKHS} for the definition of the RKHS norm). The observation noise are assumed to be i.i.d. sub-Gaussian random variables. Specifically, it is assumed that $\forall h\in \Rr, \forall t\in \Nn,
\E[e^{h\epsilon_t}]\le \exp(\frac{h^2R^2}{2}),
$ for some $R>0$.
The sub-Gaussian assumption implies that $\E[\epsilon_t] = 0$, for all $t$.

\subsection{The Information Gain and The Upper Bounds on Regret}\label{InfoGainRegBounds}

The regret analysis of Bayesian optimization algorithms typically consists of two main components. One is a bound on the maximal information gain $\gamma_T$, and the other is a confidence bound for random processes. The bound on $\gamma_T$ is treated identically under both Bayesian and frequentist settings. Confidence bounds which are utilized under each setting are different, however. 
To motivate the analysis of $\gamma_T$, we provide a sketch for the analysis of GP-UCB. Since our bounds on $\gamma_T$ are independent of the learning algorithm, they are indiscriminately applicable to all settings where the regret bound is given in terms of $\gamma_T$.


A classic approach to the sequential optimization problem is to construct a $1-\delta$ upper confidence bound for $f$, after $t-1$ observations, in the form of 
\begin{eqnarray}\label{ucb}
U_t(x) = \mu_{t-1}(x) + \beta_t(\delta)\sigma_{t-1}(x).
\end{eqnarray}
GP modelling provides us with closed form expressions for $\mu_{t-1}(x)$ and $\sigma_{t-1}(x)$, while $\beta_t(\delta)$ is a properly chosen confidence width multiplier that ensures~$U_t(x)\ge f(x)$, with probability at least $1-\delta$. Following the standard analysis~\citep[see e.g.,][]{srinivas2010gaussian}, a probability union bound implies that, with probability at least $1-\delta$, for all $t\ge 1$,
\begin{eqnarray}\nn
f(x^*) - f(x_t) &\le& \mu_{t-1}(x^*) + \beta_t(3\delta/(\pi^2t^2))\sigma_{t-1}(x^*)\\\nn
&&\hspace{-1em}- \mu_{t-1}(x_t) + \beta_t(3\delta/(\pi^2t^2))\sigma_{t-1}(x_t).
\end{eqnarray}
The selection rule of GP-UCB with $\alpha_t = \beta_t(3\delta/(\pi^2t^2))$ indicates that $\mu_{t-1}(x^*) + \alpha_t\sigma_{t-1}(x^*)\le \mu_{t-1}(x_t) + \alpha_t\sigma_{t-1}(x_t)$. Therefore, for GP-UCB, with probability at least $1-\delta$, $f(x^*) - f(x_t) \le 2\alpha_t\sigma_{t-1}(x_t)$, for all $t\ge 1$. Summing up both sides over $t$ and applying Cauchy-Schwarz inequality, we get, with probability at least $1-\delta$, 
\begin{eqnarray}\nn
R(T;\text{GP-UCB}) \le 2\alpha_T\sqrt{\df T}, 
\end{eqnarray}
where 
$\df=\sum_{t=1}^T\sigma^2_{t-1}(x_t)$.

\cite{srinivas2010gaussian} showed that, under the Bayesian setting with some mild regularity assumptions which hold for most typical kernels (e.g., SE and Mat{\'e}rn-$\nu$ with $\nu>2$), $\beta_t(\delta) = \O(\sqrt{\log(t/\delta)})$, consequently $\alpha_T=\O(\sqrt{\log(T/\delta)})$. Under the frequentist setting,~\cite{Chowdhury2017bandit} established similar confidence bounds; albeit, with a much larger width multiplier $\beta_t(\delta) = B+R\sqrt{2(\gamma_{t-1}+1+\log(1/\delta))}$ resulting in $\alpha_T = \O(\sqrt{\gamma_T})$, and consequently an $\O(\sqrt{\gamma_T})$ gap between the regret bounds, for GP-UCB, under the Bayesian and frequentist settings. It is unknown whether this confidence bound and the regret bounds for vanilla GP-UCB under the frequentist setting can be improved. The main challenge in establishing confidence bounds is the adaptivity of the observation sequence in the sequential optimization problem (in contrast to an offline setting with predetermined observation points). Of significant theoretical value, the SupKernelUCB algorithm~\citep{Valko2013kernelbandit} gets around this technicality and achieves $\O(\sqrt{\df T}\log^2(T))$ regret through the use of an independent batch observation trick, which can be attributed back to~\cite{auer2002using}. The original analysis of SupKernelUCB, which was given on a finite search space, can be extended to more general compact sets through a discretization argument, preserving $\tilde{O}(\sqrt{\df T})$ regret~\citep[see][Appendix A.4]{cai2020lower}.

It remains to bound the cumulative variance at the observation points, $\df$.
The standard approach for bounding $\df$ is to use the information gain that refers to the mutual information $I(\yb_t;\hat{f})$~\citep{cover1999elements} between $\yb_t$ and $\hat{f}$.
From the closed form expression of mutual information between two multivariate Gaussian distributions, we know that
$
I(\yb_t;\hat{f}) = \frac{1}{2}\log\det(\Ib_t+\frac{1}{\tau}K_{\Xb_t,\Xb_t}).
$
Using Jensen's inequality, \cite{srinivas2010gaussian} proved that
$\df\le c_1 I(\yb_T;\hat{f})$
where $c_1 = 2/\log(1+1/\tau)$ is an absolute constant. 

It is standard to proceed by defining a kernel-specific and $\Xb_T$-independent maximal information gain, 
\begin{eqnarray}\label{maximal}
\gamma_T = \sup_{\Xb_T\subseteq \X}I(\yb_T;\hat{f}).
\end{eqnarray}
The regret bounds are then given in terms of $\gamma_T$. 

For specific kernels (Mat{\'e}rn and SE),~\cite{srinivas2010gaussian,Janz2020SlightImprov} proved upper bounds on $\gamma_T$ which are commonly used to provide explicit regret bounds. The contribution of this paper is to derive novel bounds on $\gamma_T$ (consequently, on $\df$) which immediately translate to improved regret bounds under various settings.

\section{Upper Bounds on the Information Gain}\label{Sec:Analysis}

Our bounds on the information gain are achieved through a finite dimensional projection of the GP model in the RKHS corresponding to $k$. We
start with outlining the details of the RKHS and the finite dimensional projection of the GP model. We then present the bounds on $\gamma_T$.  

\subsection{RKHS and Mercer's Theorem}\label{RKHS}

Consider a positive definite kernel $k:\X\times\X\rightarrow \Rr$ with respect to a finite Borel measure (e.g., the Lebesgue measure) supported on $\X$. A Hilbert space $H_k$ of functions on $\X$ equipped with an inner product $\langle \cdot, \cdot \rangle_{H_k}$ is called an RKHS with reproducing kernel $k$ if the following are satisfied.
For all $x\in\X$, $k(\cdot,x)\in H_k$, and
for all $x\in\X$ and $f\in H_k$, $\langle f,k(\cdot,x)\rangle_{H_k} = f(x)$ (reproducing property).

An RKHS is completely specified with its kernel function and vice-versa. The inner product induces the RKHS norm $\|f\|^2_{H_k} = \langle f,f\rangle_{H_k}$ that can be interpreted as a measure for the complexity of $f$.

Mercer's theorem provides an alternative representation for GP kernels as an inner product of infinite dimensional feature maps~\citep[see e.g.,][Theorem 4.1]{Kanagawa2018}.

\begin{theorem}[Mercer's Theorem]\label{The:Mercer}
Let $k$ be a continuous kernel with respect to a finite Borel measure on $\X$.
There exists $\{(\lambda_m,\phi_m)\}_{m=1}^{\infty}$ such that $\lambda_m\in \Rr^{+}$, $\phi_m\in H_k$, for $m\ge1$, and
\begin{eqnarray}\nn
k(x,x') = \sum_{m=1}^{\infty} \lambda_m\phi_m(x)\phi_m(x').
\end{eqnarray}
\end{theorem}

The $\{\lambda_m\}_{m=1}^\infty$ and the $\{\phi_m\}_{m=1}^\infty$ are referred to as the eigenvalues and the eigenfeatures (or eigenfunctions) of $k$, respectively. Throughout the paper, it is assumed that $\{\lambda_m\}_{m=1}^\infty$ are in a decreasing order: $\lambda_1\ge\lambda_2\ge\dots$. 
Our technical assumption on $k$, used in the analysis of $\gamma_T$, is specified next, that is the same as in ~\cite{Bartlett2018}, and holds for practically relevant kernels~\citep[cf.][]{Gabriel2020practicalfeature}.

\begin{assumption}\label{ass1}
\textbf{a)} $k$ is a Mercer kernel (that is to satisfy the conditions of Mercer's theorem). 
\textbf{b)} $\forall x,x'\in\X$, $|k(x,x')|\le \bar{k}$, for some $\bar{k}>0$. 
\textbf{c)}
$\forall m \in \Nn, \forall x\in\X$, $|\phi_m(x)|\le \psi$, for some $\psi>0$.
\end{assumption}

As a result of Mercer's theorem, we can express a GP sample $\hat{f}$ in terms of a weight vector in the feature space of $k$
\begin{eqnarray}\label{fhat}
\hat{f}(\cdot) = \sum_{m=1}^\infty W_m\lambda^{\frac{1}{2}}_m\phi_m(\cdot),
\end{eqnarray}
where the weights $W_m$ are i.i.d. random variables with standard normal distribution \citep[see e.g.,][Remark 4.4]{Kanagawa2018}.
It is straightforward to check that $\hat{f}$ given in~\eqref{fhat} is a zero mean GP with kernel $k$. We refer to this representation as the feature space representation in contrast to the function space representation presented in~\cref{Sec:GPs}.

The RKHS can also be represented in terms of $\{(\lambda_m,\phi_m)\}_{m=1}^{\infty}$ using Mercer's representation theorem~\citep[see e.g.,][Theorem 4.2]{Kanagawa2018}.

\begin{theorem}[Mercer's Representation Theorem]\label{The:MercerRep} Let $\{(\lambda_m,\phi_m)\}_{m=1}^\infty$ be the same as in Theorem~\ref{The:Mercer}. Then, the RKHS of $k$ is given by
\begin{eqnarray}\nn
\scriptsize
&&\hspace{-2em}H_k = \\\nn
&&\hspace{-1em}\left\{ f(\cdot)=\sum_{m=1}^{\infty}w_m\lambda_{m}^{\frac{1}{2}}\phi_m(\cdot): \|f\|_{H_k} \triangleq \sum_{m=1}^\infty w_m^2<\infty \right\}.
\end{eqnarray}
\end{theorem}
Mercer's representation theorem provides an explicit definition for the RKHS norm. It also indicates that $\{\lambda_m^{\frac{1}{2}}\phi_m\}_{m=1}^\infty$ form an orthonormal basis for $H_k$.

\subsection{Projection onto a Finite Dimensional Space}\label{Auxilary}

The feature space representation of typical GP kernels is infinite dimensional. 
To overcome the difficulty of working in infinite dimensional spaces,
we use a projection $\mathcal{P}_D$ on a $D$ dimensional RKHS consisting of the first $D$ features (corresponding to the $D$ largest eigenvalues of the kernel). Specifically,
consider the $D$-dimensional feature space $\phib_D(\cdot) =[\phi_1(\cdot), \phi_2(\cdot),\dots, \phi_D(\cdot)]^{\TP}$, the $D$-dimensional column vector $\W_D = [W_1,W_2,\dots,W_D]^{\TP}$ and the diagonal matrix $\Lambda_D =\text{diag}([\lambda_1, \lambda_2, \dots, \lambda_D])$ with $[\lambda_1, \lambda_2, \dots, \lambda_D]$ as the diagonal entries. The projection of $\hat{f}$ on the $D$-dimensional space is given by
\begin{eqnarray}\nn
\mathcal{P}_D [\hat{f}(\cdot)] &=& \W_D^{\TP}\Lambda_D^{\frac{1}{2}}\phib_D(\cdot).
\end{eqnarray}

Notice that $\mathcal{P}_D [\hat{f}]$ is a zero mean GP with kernel $k_{P}(x,x')=\sum_{m=1}^D\lambda_m\phi_m(x)\phi_m(x')$. We used the subscript $P$ to signify the space resulted from the projection. In addition, let $\mathcal{P}^{\bot}_D [\hat{f}] = \hat{f} - \mathcal{P}_D [\hat{f}]$ be the orthogonal part of $\hat{f}$ with respect to the projection. Notice that $\mathcal{P}^{\bot}_D [\hat{f}]$ is also a GP, with kernel $k_O(x,x') = k(x,x')-k_P(x,x')$. We used the subscript $O$ to signify the orthogonal part.


We define the following quantity based on the tail mass of the eigenvalues of $k$
\begin{eqnarray}\label{eq:deltaM}
\delta_D =\sum_{m=D+1}^\infty\lambda_m\psi^2.
\end{eqnarray}
If $\lambda_m$ diminishes at a sufficiently fast rate (see Defenition~\ref{Def:PolExp}), $\delta_D$ becomes arbitrarily small when $D$ is large enough. For all $x,x'\in~\X$, we then have $ k_O(x,x')\le \delta_D.$

\subsection{Analysis of the Information Gain}\label{Analysisofdf}


Here, we establish a novel upper bound on $\gamma_T$.




\begin{theorem}[Bounding $\gamma_T $]\label{TheoremDT}
Consider a GP with a kernel $k$ satisfying Assumption~\ref{ass1}. For $D\in \Nn$, let $\delta_D$ be as defined in~\eqref{eq:deltaM}. 
The following upper bound on $\gamma_T $, defined in~\eqref{maximal}, holds for  all $D\in\Nn$.
\begin{eqnarray}\nn
\gamma_T \le  \frac{1}{2}D\log\left(1+\frac{\bar{k}T}{\tau D}\right) +  \frac{1}{2}\frac{\delta_DT}{\tau}.
\end{eqnarray}
\end{theorem}
The expression can be simplified as 
\begin{eqnarray}\label{simplified}
\gamma_T = \O\left(D\log( T) + \delta_DT \right).
\end{eqnarray}

In contrast to the existing results, Theorem \ref{TheoremDT} provides an upper bound in terms of the spectral properties of the GP kernel through $\delta_D$ that is applicable to all kernels based on their eigendecay. Specializing this bound for common kernels (Mat{\'e}rn and SE) results in tight upper bounds on $\gamma_T$ (up to a $\log(T)$ factor), and consequently improved bounds on regret, across various settings, compared to the existing ones.

\emph{Proof Sketch.}
Recall
$
I(\yb_t;\hat{f}) = \frac{1}{2}\log\det(\Ib_t+\frac{1}{\tau}K_{\Xb_t,\Xb_t}).
$
The problem is thus bounding the $\log\det$ of the covariance matrix $\Ib_t+\frac{1}{\tau}K_{\Xb_t,\Xb_t}$ for an arbitrary sequence $\Xb_t$ of observation points. 
To achieve this, 
we use the $D$-dimensional projection in the RKHS.
Recall $k=k_P+k_O$. Let us use the notations $K_{P,\Xb_t,\Xb_t}=[k_{P}(x_i,x_j)]_{i,j=1}^T$ and $K_{O,\Xb_t,\Xb_t}=[k_{O}(x_i,x_j)]_{i,j=1}^T$ to denote the corresponding covariance matrices. We show that $\log\det(\Ib_t+\frac{1}{\tau}K_{\Xb_t,\Xb_t})$ is bounded in terms of $\log\det(\Ib_t+\frac{1}{\tau}K_{P,\Xb_t,\Xb_t})$ and a residual term, depending on $K_{O,\Xb_t,\Xb_t}$. 
The finite dimensionality of the RKHS of $k_P$ allows us to use Weinstein–Aronszajn identity and the Gram matrix $G_t$ in the feature space of $k_P$ to bound $\log\det(\Ib_t+\frac{1}{\tau}K_{P,\Xb_t,\Xb_t})$ in terms of $\log\det(\Ib_D+\frac{1}{\tau}G_t)$. 
Elementary calculation can be used to establish a bound on the $\log\det$ of a positive definite matrix in terms of its trace. Utilizing this result, we bound $\log\det(\Ib_D+\frac{1}{\tau}G_t)$ by $\O(D\log(T))$. 
We use the bound on the $\log\det$ of a positive definite matrix in terms of its trace, again, to bound the residual term depending on $K_{O,\Xb_t,\Xb_t}$ by $\O(\delta_D T)$, taking advantage of the property that $k_{O}(x,x')\le \delta_D$.
A detailed proof is given in Appendix~\ref{AppendixA}.

\begin{remark}\label{remark1}
The quantity $\tilde{\mathcal{D}}_T=\sum_{t=1}^T\sigma_T^2(x_t)$ is often referred to as the effective dimension of the sequential optimization problem~\citep{Valko2013kernelbandit, Calandriello2019Adaptive, Janz2020SlightImprov}. To obtain an explicit bound on $\gamma_T$, we increase $D$ such that $D\log(T)$ and $T\delta_D$ on the right hand side of~\eqref{simplified} become of the same order. For such sufficiently large $D$, we have $\gamma_T=\O(D\log(T))$; consequently, $\tilde{\mathcal{D}}_T=\O(D\log(T))$, which explains the use of the term effective dimension. That is to say the behavior of the kernel
becomes similar to that of a finite $D$-dimensional kernel (up to a $\log(T)$ factor).
\end{remark}

\subsection{Conditions on the Eigendecay of the GP Kernel}


We now discuss the implications of Theorem~\ref{TheoremDT}, under conditions on the eigendecay of $k$. In particular, we define the following characteristic eigendecay profiles (which are similar to those outlined in~\cite{Bartlett2018}). 
\begin{definition}[Polynomial and Exponential Eigendecay]\label{Def:PolExp}
Consider the eigenvalues $\{\lambda_m\}_{m=1}^\infty$ of $k$ as given in Theorem~\ref{The:Mercer} in a decreasing order.

\begin{enumerate}
    \item For some $C_p>0$, $\beta_p> 1$, $k$ is said to have a $(C_p,\beta_p)$ polynomial eigendecay, if for all $m\in \Nn$, we have $\lambda_m\le C_pm^{-\beta_p}$.
    \item For some $C_{e,1},C_{e,2},\beta_e>0$, $k$ is said to have a $(C_{e,1},C_{e,2},\beta_e)$ exponential eigendecay, if for all $m\in \Nn$, we have $\lambda_m\le C_{e,1}\exp(-C_{e,2}m^{\beta_e})$.
\end{enumerate}
\end{definition}

The following corollary is a consequence of Theorem~\ref{TheoremDT}.

\begin{corollary}\label{Cor:PolExp}
Consider $\gamma_T$ defined in~\eqref{maximal}.
If $k$ has a $(C_p,\beta_p)$ polynomial eigendecay, we have
\begin{eqnarray}\nn
\gamma_T\le \left((\frac{C_p\psi^2T}{\tau})^{\frac{1}{\beta_p}}\log^{-\frac{1}{\beta_p}}(1+\frac{\bar{k}T}{\tau})+1\right)\log(1+\frac{\bar{k}T}{\tau}).
\end{eqnarray}
The expression can be simplified as $\gamma_T=\O\left(T^{\frac{1}{\beta_p}}\log^{1-\frac{1}{\beta_p}}(T)\right)$.

If $k$ has a $(C_{e,1},C_{e,2},\beta_e)$ exponential eigendecay, we have
\begin{eqnarray}\nn
\gamma_T \le \left(\left(\frac{2}{C_{e,2}}(\log(T)+C_{\beta_e})\right)^{\frac{1}{\beta_e}}+1\right)\log(1+\frac{\bar{k}T}{\tau}),
\end{eqnarray}
where $C_{\beta_e} = \log(\frac{C_{e,1}\psi^2}{\tau C_{e,2}})$ if $\beta_e=1$, and $C_{\beta_e} =    \log(\frac{2C_{e,1}\psi^2}{\tau \beta_eC_{e,2}})+(\frac{1}{\beta_e}-1)\left( \log(\frac{2}{C_{e,2}}(\frac{1}{\beta_e}-1))-1) \right)$, otherwise.
The expression can be simplified as
$\gamma_T=\O(\log^{1+\frac{1}{\beta_e}}(T))$.
\end{corollary}

Corollary~\ref{Cor:PolExp} gives general bounds on $\gamma_T$ providing the polynomial and exponential conditions on the eigendecay of $k$ are satisfied.
A detailed proof is provided in Appendix~\ref{AppendixB}. 

\begin{remark}
It is known that, in the case of a Mat{\'e}rn kernel with smoothness parameter $\nu>\frac{1}{2}$, $\lambda_m  =O(m^{-\frac{2\nu+d}{d}})$~\citep{MaternEigenvaluessantin2016}; and, in the case of SE kernel, $\lambda_m  = O(\exp(-m^{\frac{1}{d}}))$~\citep{SEEigenvalues}. Also, see~\citet{Gabriel2020practicalfeature} which gave closed form expression of their eigenvalue-eigenfeature pairs on hypercubes. Thus, as special cases of polynomial and exponential eigendecays, we have
\begin{eqnarray}\nn
\gamma_T &=& {O}\left(T^{\frac{d}{2\nu+d}}\log^{\frac{2\nu}{2\nu+d}}(T)\right),~~\text{for Mat{\'e}rn-$\nu$ kernel},\\\nn
\gamma_T &=& {O}\left(\log^{d+1}(T)\right),~~\text{for SE kernel},
\end{eqnarray}
which are tight up to $\log(T)$ factors, based on the lower bounds reported in~\cite{Scarlett2017Lower}.
\end{remark}

\section{The Improved Regret Bounds for Bayesian Optimization Algorithms}\label{Sec:RegBounds}

Utilizing the upper bounds on $\gamma_T$  established in Theorem~\ref{TheoremDT} and Corollary~\ref{Cor:PolExp}, we can derive new regret bounds for Bayesian optimization algorithms. In particular, consider the sequential optimization problem given in~\cref{SeqOptP}. Under the Bayesian setting, the application of our bounds on $\gamma_T$ to $\O(\sqrt{\log(T)\gamma_T T})$ regret bounds for GP-UCB and GP-TS established in~\cite{srinivas2010gaussian, kandasamy2018parallelised} results in $\O\left(T^{\frac{\beta_p+1}{2\beta_p}}\log^{1-\frac{1}{2\beta_p}}(T)\right)$ and $\O\left(T^{\frac{1}{2}}
\log^{1+\frac{1}{2\beta_e}}(T)
\right)$ regret bounds under the polynomial and exponential eigendecays, respectively. 
As special cases of polynomial and exponential eigendecays, the regret bound for both GP-UCB and GP-TS with Mat{\'e}rn and SE kernels is ${\O}\left(
T^{\frac{\nu+d}{2\nu+d}}\log^{\frac{4\nu+d}{4\nu+2d}}(T)
\right)$ and $\O\left(T^{\frac{1}{2}}(\log(T))^{\frac{d}{2}+1}\right)$, respectively.

Under the frequentist setting, the application of our bounds on $\gamma_T$ to $\O(\log^2(T)\sqrt{\gamma_T T})$ regret bounds for SupKernelUCB reported in~\cite{Valko2013kernelbandit}, results in $\O\left(T^{\frac{\beta_p+1}{2\beta_p}}\log^{\frac{5}{2}-\frac{1}{2\beta_p}}(T)\right)$ and $\O\left(T^{\frac{1}{2}}
\log^{\frac{5}{2}+\frac{1}{2\beta_e}}(T)
\right)$ regret bounds under the polynomial and exponential eigendecays, respectively. Consequently, the regret bounds for SupKernelUCB with Mat{\'e}rn and SE kernels are ${\O}\left(
T^{\frac{\nu+d}{2\nu+d}}\log^{\frac{5}{2}-\frac{d}{4\nu+2d}}(T)
\right)$ and $\O\left(T^{\frac{1}{2}}\log^{\frac{d}{2}+\frac{5}{2}}(T)\right)$, respectively, reducing the gap with the lower bounds reported in~\cite{Scarlett2017Lower} to a $\log^{\frac{5}{2}}(T)$ factor. Furthermore, under the frequentist setting, our bounds on $\gamma_T$ improve the $\tilde\O(\gamma_T\sqrt{ T})$ regret bounds for vanilla GP-UCB and GP-TS reported in~\cite{Chowdhury2017bandit}.

\section{ Conclusion}\label{Sec:Discussion}

We introduced a general approach to bounding the information gain in Bayesian optimization problems. 
We provided explicit bounds in terms of $T$ on $\gamma_T$ and regret, under conditions on the eigendecay of the kernel, which directly apply to common kernels such as Mat{\'e}rn and SE and show significant improvements over the state of the art. 
Our results establish the first tight regret bounds (up to $\log(T)$ factors) with the Mat{\'e}rn kernel under the frequentist setting, which shows our bound on $\gamma_T$ is tight (up to $\log(T)$ factors). 
The application of our bounds on $\gamma_T$ to numerous other settings, where the regret bounds are given in terms of $\gamma_T$, as listed in~\cref{RelatedW}, improves the best known regret bounds.

\section*{Acknowledgment}

We thank Jonathan Scarlett for insightful comments on an earlier version of this paper.

\vspace{1em}

\medskip 
\bibliography{references.bib}
\bibliographystyle{abbrvnat}

\begin{appendices}
\section{(Proof of Theorem~\ref{TheoremDT})}\label{AppendixA}

We bound $ I(\yb_t;\hat{f})=\frac{1}{2}\log\det(\Ib_t+\frac{1}{\tau}K_{\Xb_t,\Xb_t})$ for an arbitrary observation sequence $\Xb_t$. Recall $k_P$ and $k_O$ and the respective covariance matrices $K_{P,\Xb_t,\Xb_t}=[k_{P}(x_i,x_j)]_{i,j=1}^T$ and $K_{O,\Xb_t,\Xb_t}=[k_{O}(x_i,x_j)]_{i,j=1}^T$, where $k_P$ corresponds to the $D$-dimensional projection in the RKHS of $k$, and $k_O$ corresponds to the orthogonal element. Noticing $K_{\Xb_t,\Xb_t} = K_{P,\Xb_t,\Xb_t}+K_{O,\Xb_t,\Xb_t}$, we have
\begin{eqnarray}\nn
I(\yb_t;\hat{f}) &=& \frac{1}{2}\log\det (\Ib_t + \frac{1}{\tau}K_{\Xb_t,\Xb_t})\\\nn 
&=& \frac{1}{2}\log\det \left(\Ib_t + \frac{1}{\tau}(K_{P,\Xb_t,\Xb_t}+K_{O,\Xb_t,\Xb_t})\right)\\\nn 
&=&
\frac{1}{2}\log\det \left((\Ib_t + \frac{1}{\tau}K_{P,\Xb_t,\Xb_t})(\Ib_t + \frac{1}{\tau}(\Ib_t + \frac{1}{\tau}K_{P,\Xb_t,\Xb_t})^{-1}K_{O,\Xb_t,\Xb_t})\right)\\\label{twoS}
&&\hspace{-3em}=\frac{1}{2}\log\det(\Ib_t + \frac{1}{\tau}K_{P,\Xb_t,\Xb_t}) + \frac{1}{2}\log\det\left(\Ib_t + \frac{1}{\tau}(\Ib_t + \frac{1}{\tau}K_{P,\Xb_t,\Xb_t})^{-1}K_{O,\Xb_t,\Xb_t}\right),~~~~~~~~
\end{eqnarray}
where for the last line we used $\det(AB) = \det(A)\det(B)$ which holds for all two square matrices of the same dimensions. The equation~\eqref{twoS} decouples the $\log\det$ of the covariance matrix corresponding to $k$ into that of $k_P$ and a residual term depending on $k_O$. We now proceed to bounding the two terms on the right hand side of~\eqref{twoS}. 

We can upper bound the first term on the right hand side of~\eqref{twoS} using a bound on the $\log\det$ of the Gram matrix in the $D$-dimensional feature space of $k_P$. Let us define $\Phib_{t,D} = [\phib_D(x_1), \phib_D(x_2), \dots,\phib_D(x_t) ]^{\TP}$, a $t\times D$ matrix which stacks the feature vectors $\phib^{\TP}_D(x_s)$, $s=1,\dots,t$, at the observation points, as its rows. 
Notice that 
\begin{eqnarray}\nn
K_{P,\Xb_t,\Xb_t} = \Phib_{t,D}\Lambda_D\Phib_{t,D}^{\TP}.
\end{eqnarray}

Consider the Gram matrix 
\begin{eqnarray}\nn
G_t = \Lambda_D^{\frac{1}{2}}\Phib_{t,D}^{\TP}\Phib_{t,D}\Lambda_D^{\frac{1}{2}}.
\end{eqnarray}
By Weinstein–Aronszajn identity\footnote{That is a special case of matrix determinant lemma.}~\citep{pozrikidis2014introduction}
\begin{eqnarray}\label{WeinAron}
\det(\Ib_D +\frac{1}{\tau} G_t) = \det(\Ib_t + \frac{1}{\tau}K_{P,\Xb_t,\Xb_t}).
\end{eqnarray}

We can prove the following lemma on the relation between the $\log\det$ and the trace of a positive definite matrix. 
\begin{lemma}\label{logdet}
For all positive definite matrices $P\in \Rr^{n\times n}$, we have
\begin{eqnarray}\nn
\log\det(P)\le n\log( \tr (P)/n).
\end{eqnarray}
\end{lemma}
The proof is provided at the end of this section. 

We next bound the trace of $\Ib_D+\frac{1}{\tau}G_t$.
Notice that, for all $x\in \X$,
\begin{eqnarray}\nn
\|\phib_D(x)\Lambda^{\frac{1}{2}}_D\|^2_2 &=& \sum_{m=1}^D\lambda_m\phi^2_m(x)\\\nn
&=&k_{P}(x,x)\\\nn
&\le&\bar{k}.
\end{eqnarray}
Thus,
\begin{eqnarray}\nn
\tr(\Ib_D+\frac{1}{\tau} G_t) &=& D+\frac{1}{\tau}\tr\left(\sum_{s=1}^t\Lambda^{\frac{1}{2}}\phib_D(x_s)\phib^{\TP}_D(x_s)\Lambda^{\frac{1}{2}}\right)\\\nn
&=& D + \frac{1}{\tau}\sum_{s=1}^t\tr\left(\Lambda^{\frac{1}{2}}\phib_D(x_s)\phib^{\TP}_D(x_s)\Lambda^{\frac{1}{2}}\right)\\\nn
&=&D+  \frac{1}{\tau}\sum_{s=1}^t\tr\left(
\Lambda^{\frac{1}{2}}\phib^{\TP}_D(x_s)\phib_D(x_s)\Lambda^{\frac{1}{2}}
\right)\\\nn
&=&D+  \frac{1}{\tau}\sum_{s=1}^t
{\|\phib_D(x_s)\Lambda^{\frac{1}{2}}\|_2^2}
\\\nn
&\le& D+\frac{t\bar{k}}{\tau}.
\end{eqnarray}
For the first line we expanded the Gram matrix, the second line holds by distributivity of trace over sum, and the third line is a result of $\tr(AA^{\TP}) = \tr(A^{\TP}A)$ which holds for any matrix $A$.

Using Lemma~\ref{logdet} and~\eqref{WeinAron}, we have
\begin{eqnarray}\nn
\log\det(\Ib_t + \frac{1}{\tau}K_{P,\Xb_t,\Xb_t})
&=& \log\det(\Ib_D + \frac{1}{\tau}G_t)\\\nn
&\le& D\log\left(\frac{\tr(\Ib_D+\frac{1}{\tau}G_t)}{D}  \right)\\\label{twoS1}
&=&D\log(1+\frac{\bar{k}t}{\tau D}).
\end{eqnarray}

To upper bound the second term on the right hand side of~\eqref{twoS}, we use $k_{O}(x,x')\le \delta_D$. Notice that $(\Ib_t + \frac{1}{\tau}K_{P,\Xb_t,\Xb_t})^{-1}$ is a positive definite matrix whose largest eigenvalue is upper bounded by $1$. For two positive definite matrices $P_1,P_2$ with the same dimensions, we have $\tr(P_1P_2)\le \bar{\lambda}_{P_1} \tr(P_2)$ where $\bar{\lambda}_{P_1}$ is the largest eigenvalue of $P_1$ (cf.~\cite{fang1994inequalities}). Thus \begin{eqnarray}\nn
\tr\left((\Ib_t + \frac{1}{\tau}K_{P,\Xb_t,\Xb_t})^{-1}K_{O,\Xb_t,\Xb_t}\right)\le\tr(K_{O,\Xb_t,\Xb_t}).
\end{eqnarray}
Since $\forall x,x'\in \X,~k_O(x,x')\le \delta_D$, we have $\tr(K_{O,\Xb_t,\Xb_t})\le t\delta_D$. Therefore,
\begin{eqnarray}\nn
\tr\left(\Ib_t + \frac{1}{\tau}(\Ib_t + \frac{1}{\tau}K_{P,\Xb_t,\Xb_t})^{-1}K_{O,\Xb_t,\Xb_t}\right) \le t(1+\frac{1}{\tau}\delta_D).
\end{eqnarray}

Using Lemma~\ref{logdet}, we have 
\begin{eqnarray}\nn
\log\det\left(\Ib_t + \frac{1}{\tau}(\Ib_t + \frac{1}{\tau}K_{P,\Xb_t,\Xb_t})^{-1}K_{O,\Xb_t,\Xb_t}\right) &\le& t\log\left( \frac{t(1+\frac{1}{\tau}\delta_D)}{t}\right)\\\nn
&=&t\log(1+\frac{1}{\tau}\delta_D)\\\label{twoS2}
&\le&\frac{t\delta_D}{\tau},
\end{eqnarray}
where for the last line we used $\log(1+z)\le z$ which holds for all $z\in \Rr$.

Putting~\eqref{twoS},~\eqref{twoS1} and~\eqref{twoS2} together, we arrive at the following bound on the information gain. 
\begin{eqnarray}\nn
I(\yb_t;\hat{f}
) &\le& \frac{1}{2}D\log(1+\frac{\bar{k}t}{\tau D}) +  \frac{1}{2}\frac{t\delta_D}{\tau},
\end{eqnarray}
which holds for any arbitrary sequence $\Xb_t\subseteq \X$. Thus
\begin{eqnarray}\nn
\gamma_T &=& \sup_{\Xb_T\subseteq \X}I(\yb_T;\hat{f}
) \\\nn
&\le& \frac{1}{2}D\log(1+\frac{\bar{k}T}{\tau D}) +  \frac{1}{2}\frac{T\delta_D}{\tau}.
\end{eqnarray}

\begin{proof}[Proof of Lemma~\ref{logdet}]
Let $\{\kappa_m>0\}_{m=1}^n$ denote the eigenvalues of $P$. Using the inequality of arithmetic and geometric means
\begin{eqnarray}\nn
\prod_{m=1}^n\kappa_m \le (\frac{1}{n}\sum_{m=1}^n\kappa_m)^n.
\end{eqnarray}
Thus, 
\begin{eqnarray}\nn
\log\det(P) &=& \log\left(\prod_{m=1}^n\kappa_m
\right)\\\nn
&\le& \log\left((\frac{1}{n}\sum_{m=1}^n\kappa_m)^n
\right)\\\nn
&=&\log\left((\frac{\tr(P)}{n})^n
\right)\\\nn
&=&n\log\left(\frac{\tr(P)}{n}
\right).
\end{eqnarray}

\end{proof}
\section{(Proof of Corollary~\ref{Cor:PolExp})}\label{AppendixB}

Under the $(C_p,\beta_p)$ polynomial eigendecay condition, the following bound on $\delta_D$ is straightforwardly derived from the decay rate of $\lambda_m$. 
\begin{eqnarray}\nn
\delta_D &=& \sum_{m=D+1}^\infty \lambda_m\psi^2\\\nn
&\le& \sum_{m=D+1}^\infty C_pm^{-\beta_p}\psi^2\\\nn
&\le& \int_{z=D}^\infty  C_pz^{-\beta_p}\psi^2dz\\\nn
&=& C_pD^{1-\beta_p}\psi^2.
\end{eqnarray}

We select $D = \lceil(C_p\psi^2T)^{\frac{1}{\beta_p}}\tau^{-\frac{1}{\beta_p}}\log^{-\frac{1}{\beta_p}}(1+\frac{\bar{k}T}{\tau})\rceil$ which is the smallest $D$ ensuring $\frac{T\delta_D}{\tau}\le D\log(1+\frac{\bar{k}T}{\tau}) $; thus, resulting in the lowest growth rate of $\gamma_T$ based on Theorem~3, which implies $\gamma_T\le D\log(1+\frac{\bar{k}T}{\tau})$
\begin{eqnarray}\nn
\gamma_T\le \left((C_p\psi^2T)^{\frac{1}{\beta_p}}\tau^{-\frac{1}{\beta_p}}\log^{-\frac{1}{\beta_p}}(1+\frac{\bar{k}T}{\tau})+1\right)\log(1+\frac{\bar{k}T}{\tau}).
\end{eqnarray}

Under the $(C_{e,1},C_{e,2},\beta_e)$ exponential eigendecay condition,
\begin{eqnarray}\nn
\delta_D  &=& \sum_{m=D+1}^\infty \lambda_m\psi^2\\\nn
&\le& \sum_{m=D+1}^\infty C_{e,1}\exp(-C_{e,2}m^{\beta_{e}})\psi^2\\\nn
&\le& \int_{z=D}^\infty C_{e,1}\exp(-C_{e,2}z^{\beta_e})\psi^2dz.
\end{eqnarray}
Now, consider two different cases of $\beta_e=1$ and $\beta_e\neq 1$. When $\beta_e=1$,
\begin{eqnarray}\nn
\int_{z=D}^\infty \exp(-C_{e,2}z^{\beta_e})dz
&=&\int_{z=D}^\infty \exp(-C_{e,2}z)dz\\\nn
&=&\frac{1}{C_{e,2}}\exp(-C_{e,2}D).
\end{eqnarray}
 
When $\beta_e\neq1$, we have

\begin{eqnarray}\nn
\int_{z=D}^\infty \exp(-C_{e,2}z^{\beta_e})dz &=& \frac{1}{{\beta_e}}\int_{z=D^{{\beta_e}}}^\infty z^{\frac{1}{{\beta_e}}-1}\exp(-C_{e,2}z)dz\\\nn
&=& \frac{1}{{\beta_e}}\int_{z=D^{{\beta_e}}}^\infty z^{\frac{1}{{\beta_e}}-1}\exp(-C_{e,2}\frac{z}{2})\exp(-C_{e,2}\frac{z}{2})dz\\\nn
&\le&\frac{1}{{\beta_e}}\int_{z=D^{{\beta_e}}}^\infty (\frac{2}{C_{e,2}}(\frac{1}{{\beta_e}}-1))^{\frac{1}{{\beta_e}}-1}\exp(-(\frac{1}{{\beta_e}}-1))\exp(-C_{e,2}\frac{z}{2})dz\\\nn
&=& \frac{2}{C_{e,2}{\beta_e}}(\frac{2}{C_{e,2}}(\frac{1}{{\beta_e}}-1))^{\frac{1}{{\beta_e}}-1}\exp(-(\frac{1}{{\beta_e}}-1))\exp(-C_{e,2}\frac{D^{{\beta_e}}}{2}).
\end{eqnarray}

The first equality is obtained by a change of parameter. 
The inequality holds since
\begin{eqnarray}
\max_{z\in\Rr} z^{\frac{1}{{\beta_e}}-1}\exp(-C_{e,2}\frac{z}{2})= (\frac{2}{C_{e,2}}(\frac{1}{{\beta_e}}-1))^{\frac{1}{{\beta_e}}-1}\exp(-(\frac{1}{{\beta_e}}-1))
\end{eqnarray}
which can be verified using the standard method of equating the derivative of the left hand side to zero.  

With a similar logic to the polynomial eigendecay case,  
when $\beta_e=1$, we select 
\begin{eqnarray}\nn
D = \lceil \frac{1}{C_{e,2}}\log(\frac{C_{e,1}\psi^2T}{\tau C_{e,2}}) \rceil.
\end{eqnarray}

When $\beta_e\neq 1$, we select
\begin{eqnarray}\nn
D = \left\lceil\left(\frac{2}{C_{e,2}}
\left(
\log(T) + \log(\frac{2C_{e,1}\psi^2}{\tau\beta_eC_{e,2}})+(\frac{1}{\beta_e}-1)\left( \log(\frac{2}{C_{e,2}}(\frac{1}{\beta_e}-1))-1) \right)
\right)
\right)^{\frac{1}{\beta_e}}
\right\rceil.
\end{eqnarray}
Theorem~3 implies
\begin{eqnarray}\nn
\gamma_T \le \left(\left(\frac{2}{C_{e,2}}(\log(T)+C_{\beta_e})\right)^{\frac{1}{\beta_e}}+1\right)\log(1+\frac{\bar{k}T}{\tau}),
\end{eqnarray}
$C_{\beta_e} = \log(\frac{C_{e,1}\psi^2}{\tau C_{e,2}})$ when $\beta_e=1$, and $C_{\beta_e} =    \log(\frac{2C_{e,1}\psi^2}{\tau\beta_eC_{e,2}})+(\frac{1}{\beta_e}-1)\left( \log(\frac{2}{C_{e,2}}(\frac{1}{\beta_e}-1))-1) \right)$, otherwise.
\end{appendices}

\end{document}